\documentclass[letterpaper, 10 pt, conference]{ieeeconf}  

\usepackage{amsmath,amssymb,amsfonts}
\usepackage{bm}
\usepackage{cite}
\usepackage{algorithm}
\usepackage{algorithmic}
\usepackage{graphicx}
\usepackage{subfig}
\usepackage{caption}

\usepackage{hyperref}
\usepackage{mathtools, cuted}
\usepackage{stfloats}
\usepackage{cleveref}
\usepackage{balance}
\usepackage{enumerate}
\usepackage{amsthm}
\usepackage{amsmath}
\usepackage{amssymb}
\usepackage{color}
\usepackage{booktabs}
\usepackage{lipsum}
\usepackage{setspace}

\newcommand{\RNum}[1]{\uppercase\expandafter{\romannumeral #1\relax}}

\newtheorem{lemma}{Lemma}
\newtheorem{example}{Example}
\newtheorem{theorem}{Theorem}

\newtheorem{assumption}{Assumption}
\newtheorem{problem}{Problem}
\newtheorem{proposition}{Proposition}

\newtheorem{remark}{Remark}

\usepackage{bm}
\usepackage{threeparttable}
\usepackage{mathrsfs} 
 
\newcommand{\mK}{{\mathsf{K}}}
\newcommand{\mH}{{\mathsf{H}}}
\newcommand{\tr}{{{\mathsf T}}}

\IEEEoverridecommandlockouts                              

\title{\LARGE \bf
On the Optimization Landscape of Dynamic Output Feedback: A Case Study for Linear Quadratic Regulator
}

\author{Jingliang Duan, Wenhan Cao, Yang Zheng, Lin Zhao
\thanks{J. Duan and L. Zhao  are with the Department of Electrical and Computer Engineering, National University of Singapore, Singapore. {\tt\small Email: (duanjl,elezhli)@nus.edu.sg}.
}
\thanks{J. Duan is also with the School of Mechanical Engineering, University of Science and Technology Beijing, China. {\tt\small Email:duanjl@ustb.edu.cn}.
}
\thanks{W. Cao is with the School of Vehicle and Mobility, Tsinghua University, Beijing, 100084, China. {\tt\small Email: cwh19@mails.tsinghua.edu.cn}.
}
\thanks{Y. Zheng is with the Department of Electrical and Computer Engineering, University of California San Diego, USA. {\tt\small Email: zhengy@eng.ucsd.edu}.
}
\thanks{Corresponding author: L. Zhao
}
}

\begin{document}

\maketitle
\thispagestyle{empty}
\pagestyle{empty}

\begin{abstract}
The convergence of policy gradient algorithms in reinforcement learning hinges on the optimization landscape of the underlying optimal control problem. Theoretical insights into these algorithms can often be acquired from analyzing those of linear quadratic control. However, most of the existing literature only considers the optimization landscape for static full-state or output feedback policies (controllers). We investigate the more challenging case of dynamic output-feedback policies for linear quadratic regulation (abbreviated as \texttt{dLQR}), which is prevalent in practice but has a rather complicated optimization landscape. We first show how the \texttt{dLQR} cost varies with the coordinate transformation of the dynamic controller and then derive the optimal transformation for a given observable stabilizing controller. At the core of our results is the uniqueness of the stationary point of \texttt{dLQR} when it is observable, which is in a concise form of an observer-based controller with the optimal similarity transformation. These results shed light on designing efficient algorithms for general decision-making problems with partially observed information.
\end{abstract}


%
\IEEEpeerreviewmaketitle

\section{Introduction}
Reinforcement learning (RL) aims to directly learn optimal policies that minimize long-term cumulative costs through interacting with unknown environments. The past few years have witnessed great successes of RL in various domains, such as video games \cite{mnih2015human}, robots control \cite{nguyen2019review_RLrobot}, nuclear fusion \cite{degrave2022magnetic}, and recommender systems \cite{zou2019recommender}. Despite the impressive empirical performance of many policy gradient algorithms (such as DDPG \cite{lillicrap2015DDPG}, SAC \cite{Haarnoja2018SAC}, DSAC \cite{duan2021distributional}),  theoretical guarantees of their convergence, optimality, sample complexity, etc., remain unexplored and a big challenge.

As a case study, canonical optimal control of linear time-invariant (LTI) systems have been commonly analyzed to help reveal various theoretical properties of policy gradient methods \cite{fazel2018global,bu2019lqr,mohammadi2019CT-LQR,malik2019derivative}. 
For example, recent investigations from the learning perspective show that the cost function of LQR enjoys an interesting property of gradient dominance~\cite{fazel2018global,mohammadi2019CT-LQR}. This enables a global linear convergence characterization for a variety of model-based and model-free gradient descent methods for solving LQR, such as policy gradient and actor-critic algorithms~\cite{fazel2018global,yang2019global}, despite the non-convexity of optimizing the quadratic cost over the control gain. 
In addition, a series of subsequent studies examined the gradient dominance property for optimal control in different settings, including finite-horizon noisy LQR \cite{hambly2020noisyLQR}, LQR tracking \cite{ren2021lqrtracking},  Markovian jump LQR \cite{jansch2020Mjump}, linear $\mathcal{H}_2$ control with $\mathcal{H}_{\infty}$ constraints \cite{zhang2020robust}, finite MDPs \cite{bhandari2019global}, and risk-constrained LQR \cite{zhao2021primal}.

\vspace{-1.5pt}

The aforementioned literature mainly focuses on the case of static full state-feedback control. In many practical settings, the complete state information of the underlying system may not be directly available. Some recent works have studied static output-feedback (SOF) controllers to optimize a linear quadratic cost function~\cite{duan2022optimization,fatkhullin2021CTSOF,feng2020connectivity,bu2019topological}. Different from the full state-feedback LQR, it is shown that policy gradient methods for solving optimal SOF controllers do not possess the gradient dominance property and thus are unlikely to find the globally optimal solution.  In fact, the set of stabilizing SOF controllers is typically disconnected, and the stationary points can be local minima, saddle points, or even local maxima \cite{fatkhullin2021CTSOF,feng2020connectivity}. Moreover, even finding a stabilizing SOF controller is a challenging task~\cite{blondel1997np,syrmos1997static}. 

This paper takes a step further to analyze the optimization landscape of the infinite-horizon dynamic output-feedback LQR (\texttt{dLQR}). From classical control theory, a \textit{stabilizing} dynamic controller for~\texttt{dLQR} can be found via designing separately a stable observer and a state-feedback controller thanks to the separation principle~\cite{lewis2012optimal}. In the context of reinforcement learning, an observer-based dynamic controller can be learned through gradient descent optimization of the LQR cost. The very recent work~\cite{mohammadi2021lack} showed that gradient dominance condition does not hold for learning observer-based dynamic controller, and their analysis assumes complete knowledge of the system model. In contrast, we consider a model free setting, where we assume that a general full-order dynamic controller is learned directly from the LQR cost. The recent closely related work~\cite{tang2021analysis,zheng2021analysis} analyzed the structure of optimal dynamic controllers for the classical Linear Quadratic Gaussian (LQG) control problem. It was found that all stationary points that correspond to minimal controllers (i.e., whose state-space realization is reachable and observable) are globally optimal to LQG, and that these stationary points are identical up to coordinate (similarity) transformations. Different from LQG which considers stochastic linear systems and minimizes a limiting \textit{average} cost (or the variance of the steady state), the \texttt{dLQR} seeks a dynamic controller that minimizes an infinite-horizon \textit{accumulated} cost for a deterministic linear system. In the latter case, both the system transient dynamics induced by the initial system and controller states and the similarity transformation influence the cost, which suggests a more complicated optimization landscape. Notably, the existing analysis of LQG~\cite{tang2021analysis,zheng2021analysis,mohammadi2021lack} does not extend to the \texttt{dLQR} directly. Indeed, little is known about the optimality of the converged solutions of policy gradient methods for \texttt{dLQR}. 

In this paper, we provide a comprehensive analysis of the influence of similarity transformation and the structure of the stationary points. Specifically,
\begin{enumerate}
    \item We analyze the impact of similarity transformations on the \texttt{dLQR} cost and derive an explicit form of the unique optimal similarity transformation for a given observable stabilizing controller. 
    \item We characterize the unique observable stationary point of the \texttt{dLQR} cost, which is in a concise form of an observer-based controller with the optimal similarity transformation. 
\end{enumerate}


The remainder of this paper is organized as follows.  Section \ref{sec:preliminary} presents the problem statement of the \texttt{dLQR} problem, and Section \ref{sec.formulation} derives an analytical form of the \texttt{dLQR} cost as a function of dynamic controller parameters. Section \ref{sec.similarity_transformation} analyzes the impact of similarity transformations on the \texttt{dLQR} cost. Section \ref{sec.mainresults} characterizes the structure of the observable stationary controller. The paper is concluded in Section \ref{sec:conclusion}.

\textbf{Notation:} We use $\mathbb{N}$ to denote the set of natural numbers. Given a matrix $X \in \mathbb{R}^{n \times n}$, $\rho(X)$, ${\rm Tr}(X)$, $\lambda_{\rm min}(X)$, and $\|X\|_F$ denote its spectral radius, trace, minimum eigenvalue, and Frobenius norm, respectively. $\mathbb{S}^n_{+}$ (respectively, $\mathbb{S}^n_{++}$) denotes the set of symmetric $n \times n$ positive semidefinite (respectively, positive definite) matrices. Finally, $\mathrm{GL}_n$ denotes the set of $n\times n$ invertible matrices, and $I_n$ denotes the identity matrix. 

\section{Problem Statement}
\label{sec:preliminary}
In this section, we start with the canonical linear quadratic optimal control problem, 
and then present the dynamic output-feedback Linear Quadratic Regulator (\texttt{dLQR}). 

\subsection{Linear Quadratic Control}
Consider a discrete-time linear time-invariant (LTI) system 
\begin{equation}  
\label{eq.statefunction}
\begin{aligned}
x_{t+1} &= Ax_t+Bu_t,\\
y_t &= Cx_t,
\end{aligned}
\end{equation}
where $A\in \mathbb{R}^{n\times n}$, $B\in \mathbb{R}^{n\times m}$, $C\in \mathbb{R}^{d\times n}$ are system~matrices, and $x_t\in \mathbb{R}^n$, $u_t\in \mathbb{R}^m$,  $y_t\in \mathbb{R}^d$ are the system state, input, and output measurements at time $t$, respectively. The linear quadratic control seeks a sequence $u_0, u_1, \ldots, u_t, \ldots $ minimizing the infinite-horizon accumulated cost:
\begin{equation}   
\label{eq.objective}
\begin{aligned}
\min_{u_t} \;\; &\mathbb{E}_{x_0\sim \mathcal{D}}\left[\sum_{t=0}^{\infty}\left(x_t^\tr Qx_t+u_t^\tr Ru_t\right) \right] \\
\text{subject to}\;\;&~\eqref{eq.statefunction},
\end{aligned}
\end{equation}
where $Q \in \mathbb{S}_+^{n\times n}$ and $R \in \mathbb{S}_{++}^{m\times m}$ are performance weights, the initial state $x_0$ is randomly distributed according to a given distribution $\mathcal{D}$, and the control input $u_t$ at time $t$ is allowed to depend on the historical outputs $y_{0}, y_1, \ldots,y_t$ and inputs $u_{0}, u_1, \ldots, u_{t-1}$. The initial state distribution $\mathcal{D}$ has been commonly introduced to model the randomness of initial states \cite{fazel2018global,  lee2018primal,malik2019derivative} in a data-driven learning setting. For problem \eqref{eq.objective}, the following assumption is standard: 
\begin{assumption}
\label{assumption.control_observe}
$(A,B)$ is controllable, and $(C, A)$ and $(Q^{\frac{1}{2}},A)$ are observable.
\end{assumption}

Without loss of generality, we assume $C$ has full row rank.  The state-feedback LQR corresponds to $C=I_n$. In this special case, the globally optimal controller is a static linear feedback $u_t = K x_t$, where $K \in \mathbb{R}^{m \times n}$ can be obtained via solving a Riccati equation~\cite{bertsekas2017dynamic}. In general cases where $\text{rank}(C)<n$, a static output-feedback (SOF) gain $u_t = Ky_t$ with $K \in \mathbb{R}^{m \times 
d}$ is typically insufficient to obtain good control performance. In fact, the set of stabilizing SOF gains can be highly disconnected \cite{feng2020connectivity}, and even finding a stabilizing SOF controller is generally a challenging task~\cite{blondel1997np,syrmos1997static}. Unlike SOF control, under Assumption \ref{assumption.control_observe}, a stabilizing dynamic output controller always exists and can be found easily, thanks to the well-known separation principle \cite{lewis2012optimal}. 
\begin{remark}[Observer-based controllers] \label{remark:observer}
In classical control, the following standard observer-based controller can be designed to ensure a finite cost of~\eqref{eq.objective}
\begin{equation} \label{eq:observer-based-controller}
\begin{aligned}
    \xi_{t+1} &= (A - BK-LC) \xi_t + Ly_{t} \\
    u_t &= -K \xi_{t},
\end{aligned}
\end{equation}
where $K \in \mathbb{R}^{m \times n}, L \in \mathbb{R}^{n \times d}$ are the feedback gain and observer gain matrices such that $A - BK$ and $A - LC$ are stable~\cite{zhang2020robust}. 
\hfill $\square$
\end{remark}

\subsection{The \texttt{dLQR} Problem}
\label{sec.setting}
More generally, we consider the class of full-order dynamic output-feedback controllers in the form of\footnote{This is in the standard form of strictly proper dynamic controllers, where there is no direct feed-through of $y_t$ to $u_t$ \cite{tang2021analysis,zheng2021analysis,van2020data}.}
\begin{equation}   
\label{eq.dynamic_controller}
\begin{aligned}
\xi_{t+1} &= A_{\mK}\xi_t + B_{\mK}y_t,\\
u_t &= C_{\mK}\xi_t,
\end{aligned}
\end{equation}
where $\xi_t \in \mathbb{R}^n$ is the internal state of the controller, and matrices $C_{\mK}\in \mathbb{R}^{m\times n}$, $B_{\mK} \in \mathbb{R}^{n\times d}$, $A_{\mK} \in \mathbb{R}^{n\times n}$ are the controller parameters to be learned. The observer-based controller \eqref{eq:observer-based-controller} is a special case of \eqref{eq.dynamic_controller}. Note that the controller parameterization in \eqref{eq.dynamic_controller} does not explicitly rely on the knowledge of system parameters $A$, $B$, and $C$, which allows for model-free policy learning. 

In addition to $A_{\mK}$, $B_{\mK}$, and $C_{\mK}$, the transient behavior induced by the initial controller state (or initial state estimate) also affects the accumulated cost. Let $\xi_0$ be the initial state estimate and suppose $(x_0,\xi_0)$ follows a joint distribution $\bar{\mathcal{D}}$. The dynamic output-feedback LQR (\texttt{dLQR}) which aims to minimize the accumulated linear quadratic cost \cite{modares2016optimal,rizvi2018output,rizvi2019reinforcement,rizvi2020output} is given by
\begin{equation}   
\label{eq.dLQR}
\begin{aligned}
\min_{A_{\mK},B_{\mK},C_{\mK}} \;\; &\mathbb{E}_{(x_0,\xi_0)\sim \bar{\mathcal{D}}}\left[\sum_{t=0}^{\infty}\left(x_t^\tr Qx_t+u_t^\tr Ru_t\right) \right] \\
\text{subject to}\;\;&~\eqref{eq.statefunction},~\eqref{eq.dynamic_controller}.
\end{aligned}
\end{equation}


\section{Optimization formulation of the \texttt{dLQR} Problem}
\label{sec.formulation}

In this section, we derive the analytical form of the cost function ~\eqref{eq.dLQR} in terms of the dynamic controller parameters, which is needed for analyzing its optimization landscape. 

We start by combining \eqref{eq.dynamic_controller} with  \eqref{eq.statefunction} and get the closed-loop system
\begin{equation}   
\label{eq.closed-loop-system}
\begin{bmatrix}
     x_{t+1}\\
     \xi_{t+1} 
\end{bmatrix} =  \begin{bmatrix}
     A & BC_{\mK}\\
     B_{\mK}C & A_{\mK}
\end{bmatrix}\begin{bmatrix}
     x_{t}\\
     \xi_{t} 
\end{bmatrix}.
\end{equation}
We further denote 
$$
    \bar{x}_t := \begin{bmatrix}
x_t \\
\xi_t
\end{bmatrix}, \; \bar{A} := \begin{bmatrix}
     A & 0\\
     0 & 0
\end{bmatrix}, \; \bar{B} :=  \begin{bmatrix}
     B & 0 \\
     0 & I_n
\end{bmatrix}, \bar{C}:= \begin{bmatrix}
     C & 0\\
     0 & I_n
\end{bmatrix},
$$
and write the controller parameters in a compact form
$$\mK := \begin{bmatrix}
0_{m\times d} & C_{\mK} \\
B_{\mK} & A_{\mK}
\end{bmatrix}. 
$$
Then \eqref{eq.closed-loop-system} can be expressed as
\begin{equation}
\label{eq.closed-loop-system_short}
\bar{x}_{t+1} = (\bar{A}+\bar{B} \mK \bar{C})\bar{x}_{t}.
\end{equation}
The set $\mathbb{K}$ of all stabilizing controllers is given by
\begin{equation}
\label{eq:stabilizing-K}
\mathbb{K}:=\left\{\mK = \begin{bmatrix}
0_{m\times d} & C_{\mK} \\
B_{\mK} & A_{\mK}
\end{bmatrix}:\rho(\bar{A}+\bar{B}\mK\bar{C})<1\right\}.
\end{equation}
It is known that $\mathbb{K}$ is non-convex but has at most two disconnected components~\cite{tang2021analysis,zheng2021analysis}. Upon denoting 
$$\bar{Q}=\begin{bmatrix}
     Q & 0\\
     0 & 0
\end{bmatrix}, \; F=[0,I_n],$$ 
the \texttt{dLQR} problem \eqref{eq.dLQR} can be written as
\begin{equation}  
\label{eq.objective_with_K}
\begin{aligned}
\min_{\mK} \quad&  \mathop{\mathbb{E}}_{\bar{x}_0\sim \bar{\mathcal{D}}}\left[\sum_{t=0}^{\infty}\bar{x}_t^\tr \left(\bar{Q}+F^\tr C_{\mK}^\tr R C_{\mK} F \right)\bar{x}_t \right]\\
\text{subject to} \quad & \eqref{eq.closed-loop-system_short}, \; \mK\in \mathbb{K}.
\end{aligned}
\end{equation}

For the LTI system \eqref{eq.closed-loop-system_short}, the value function of state $\bar{x}$ under a stabilizing controller $\mK \in \mathbb{K}$ takes a quadratic form as
\begin{equation}
\nonumber
V_{\mK}(\bar{x}_t): = \bar{x}_t^\tr P_{\mK}\bar{x}_t,
\end{equation}
where $P_{\mK} \in \mathbb{S}_{+}^{2n}$. Define the state correlation matrix under a stabilizing controller $\mK \in \mathbb{K}$ as
\begin{equation}
\nonumber
\Sigma_{\mK}:=\mathbb{E}_{\bar{x}_0\sim \bar{\mathcal{D}}}\sum_{t=0}^{\infty}\bar{x}_t \bar{x}_t^\tr.
\end{equation}
For each $\mK \in \mathbb{K}$, with $P_{\mK}$ and $\Sigma_{\mK}$, it is well known that the \texttt{dLQR} cost value in \eqref{eq.objective_with_K} can be computed in the following lemma \cite{tang2021analysis,fazel2018global}. 
\begin{lemma}
\label{lemma.dLQR_cost}
Given each $\mK \in \mathbb{K}$, the \texttt{dLQR} cost value is 
\begin{equation}   
\label{eq.cost_in_P}
J(\mK) = {\rm Tr}(P_{\mK}X)= {\rm Tr}\left(\begin{bmatrix}
   Q  & 0 \\
   0  &  C_{\mK}^\tr RC_{\mK}
\end{bmatrix}\Sigma_{\mK}\right),
\end{equation}
where $P_{\mK}$ and $\Sigma_{\mK}$ are the unique positive semidefinite solutions to the following Lyapunov equations
\begin{subequations}
\begin{align} 
\label{eq.lyapunov_equation}
P_{\mK} &= \bar{Q} + F^\tr C_{\mK}^\tr RC_{\mK}F \\
&\qquad \qquad +(\bar{A}+\bar{B}\mK\bar{C})^\tr P_{\mK}(\bar{A}+\bar{B}\mK\bar{C}),  \nonumber  \\
\Sigma_{\mK} &= X +(\bar{A}+\bar{B}\mK\bar{C})\Sigma_{\mK}(\bar{A}+\bar{B}\mK\bar{C})^\tr, \label{eq.lyapunov_equation_sigma}
\end{align} 
\end{subequations}
with $X =  \mathbb{E}_{\bar{x}_0\sim \bar{\mathcal{D}}} \; \bar{x}_0\bar{x}_0^\tr$. 
\end{lemma}
Note that $P_{\mK}$ can be partitioned into four $n\times n$ matrices:
\begin{equation} \label{eq:Pk-partition}
P_{\mK}=\begin{bmatrix}
  P_{\mK,11}  &  P_{\mK,12}\\
  P_{\mK,12}^\tr  & P_{\mK,22} 
\end{bmatrix}.
\end{equation}
Similarly, one has
\begin{equation} \label{eq:Sigma-X-partition}
\Sigma_{\mK}=\begin{bmatrix}
  \Sigma_{\mK,11}  &  \Sigma_{\mK,12}\\
  \Sigma_{\mK,12}^\tr  & \Sigma_{\mK,22} 
\end{bmatrix}, \;\; X=\begin{bmatrix}
  X_{11}  &  X_{12}\\
  X_{12}^\tr  & X_{22} 
\end{bmatrix}.
\end{equation}  
Finally, we formulate the \texttt{dLQR} problem \eqref{eq.dLQR} into the following optimization form.
\begin{problem}[Policy optimization for \texttt{dLQR}]
\label{pro.dLQR}
\begin{equation}  
\nonumber
\begin{aligned}
\min_{\mK} \quad&  J(\mK)\\
\text{subject to} \quad &\mK\in \mathbb{K}.
\end{aligned}
\end{equation}
where $J(\mK)$ is defined in~\eqref{eq.cost_in_P} and $\mathbb{K}$ is given in~\eqref{eq:stabilizing-K}. 
\end{problem}
Note that the initial estimate is sampled from a fixed initial distribution, and thus the matrix $X$ in~\eqref{eq.lyapunov_equation_sigma} is independent of the parameters $\mK$. 

Next, we will characterize several important properties that delineate the optimization landscape (such as the influence of similarity transformation and structure of stationary points) of policy gradient methods for solving Problem \ref{pro.dLQR}. The detailed proofs are provided in the Appendix.

\section{\texttt{dLQR} Cost under Different Similarity Transformations}
\label{sec.similarity_transformation}

For dynamic controllers, a widely used concept is the so-called \textit{similarity transformation}~\cite{zhou1996robust}. It is well-known that 
similarity transformations do not change the control performance of the LQG problem~\cite[Lemma 4.1]{zheng2021analysis}. 
However, in this section, we will show that the \texttt{dLQR} cost varies with different similarity transformations due to the transient behavior induced by initial controller states, and thus the optimization landscape of \texttt{dLQR} is distinct from LQG.

\subsection{Varying \texttt{dLQR} cost}

Given a controller $\mK$ and an invertible matrix $T\in \mathrm{GL}_n$, we define the similarity transformation on $\mK$ by
\begingroup
\def\arraystretch{0.95}
\setlength\arraycolsep{2.5pt}
\begin{equation} \label{eq:similarity-transformation}
\begin{aligned}
\mathscr{T}_T(\mK) \!= \!\begin{bmatrix}
I_m & 0 \\
0 & T
\end{bmatrix}\mK\begin{bmatrix}
I_d & 0 \\
0 & T
\end{bmatrix}^{-1}\!\!
=\!\begin{bmatrix}
0 & C_{\mK}T^{-1} \\
TB_{\mK} & TA_{\mK}T^{-1}
\end{bmatrix}.
\end{aligned}
\end{equation}
\endgroup
It is not hard to verify that if $\mK\in \mathbb{K}$ and $T\in \mathrm{GL}_n$, we have $\mathscr{T}_T(\mK) \in \mathbb{K}$; see~\cite[Lemma 3.2]{zheng2021analysis} for further discussions. 

Our first result reveals that the \texttt{dLQR} cost is not invariant w.r.t. the similarity transformation~\eqref{eq:similarity-transformation}. Indeed, we have the following result. 
\begin{proposition}
\label{prop.similarity_variance}
Let $\mK\in \mathbb{K}$ and $T\in \mathrm{GL}_n$. We have
\begin{equation}
\label{eq.cost_transformation}
J(\mathscr{T}_T(\mK))  = {\rm Tr}\left(P_{\mK}\bar{T}^{-1}X\bar{T}^{-\tr}\right),
\end{equation}
where $\bar{T}=\begin{bmatrix}
I_n & 0 \\
0 & T
\end{bmatrix}$, $P_{\mK}$ is the unique positive semidefinite solution to~\eqref{eq.lyapunov_equation}, and $X =  \mathbb{E}_{\bar{x}_0\sim \bar{\mathcal{D}}} \; \bar{x}_0\bar{x}_0^\tr$. 
\end{proposition}


Although any similarity transformation corresponds to the same transfer function in
the frequency domain, Proposition \ref{prop.similarity_variance} shows that the \texttt{dLQR} cost varies with different similarity transformations. This result is reasonable considering the facts that the initial controller state $\xi_0$ is assumed to follow a fixed distribution and that the similarity transformation implies a coordinate change of the internal controller state. If the controller coordinate changes while its initial state does not change, this essentially leads to a different dynamic controller~\eqref{eq.dynamic_controller}, which naturally results in a different \texttt{dLQR} cost value. 

\subsection{Optimal similarity transformation}
One natural consequence of Proposition \ref{prop.similarity_variance} is that for each stabilizing controller $\mK\in \mathbb{K}$, there might exist an optimal similarity transformation matrix $T^\star$  
in the sense that 
\begin{equation}
\label{eq.optimal_T}
J(\mathscr{T}_{T^\star}(\mK)) \leq J(\mathscr{T}_T(\mK)), \quad \forall T\in \mathrm{GL}_n. 
\end{equation}
In this case, we call $T^*$ 
the optimal similarity transformation matrix
 of $\mK$. 

In this paper, we refer to \eqref{eq.dynamic_controller} as an observable  controller if $(C_{\mK},A_{\mK})$ is observable. We denote the set of observable controllers as 
\begin{equation}
\nonumber
\mathbb{K}_o := \left\{ \begin{bmatrix}
0_{m\times d} & C_{\mK} \\
B_{\mK} & A_{\mK}
\end{bmatrix}: (C_{\mK},A_{\mK}) \text{ is observable} \right\}.
\end{equation}
Our next result characterizes the structure of the optimal similarity transformation for an observable stabilizing controller. 
\begin{theorem}
\label{theorem.optimal_tansformation}
Suppose $X \succ 0$ and $\mK\in \mathbb{K} \cap \mathbb{K}_o$. If the optimal transformation matrix $T^\star \in \mathrm{GL}_n$ satisfying~\eqref{eq.optimal_T} exists, it is unique and in the form of
\begin{equation}
\label{eq.optimal_tranforsmation}
T^\star=- X_{22}X_{12}^{-1}P_{\mK,12}^{-\tr}P_{\mK,22},
\end{equation}
where $P_\mK$, partitioned as~\eqref{eq:Pk-partition}, is the unique positive definite solution to \eqref{eq.lyapunov_equation}.
\end{theorem}

Theorem \ref{theorem.optimal_tansformation} identifies the form of the optimal similarity transformation, which is unique if it exists. This implies that if the optimal controller for Problem \ref{pro.dLQR} is observable, it may
be unique and be expressed as an optimal similarity transformation of a particular dynamic controller. However, the optimal similarity transformation may not always exist since $X_{12}$ can be singular. We give such an analytical example in \Cref{appendix.non-existence}.

We conclude this section by providing two examples to illustrate the impact of similarity transformation on the \texttt{dLQR} cost. 
\begin{example}
\label{example:1}
Consider an open-loop unstable dynamic system \eqref{eq.statefunction} with 
$$A=1.1, \; B=1, \; C=1, \; Q=5, \; R=1.$$ 
According to~\cite[Theoerem D.4, Example 11]{zheng2021analysis}, the set of stabilizing controllers $\mathbb{K}$ for this system has two disconnected components.  To define \texttt{dLQR}~\eqref{eq.dLQR}, we choose  
\begin{equation} \label{eq:exampleX}
X=\mathbb{E}_{\bar{x}_0\sim \bar{\mathcal{D}}} \; \bar{x}_0\bar{x}_0^\tr = \begin{bmatrix}
1&0.25\\
0.25&1
\end{bmatrix}.
\end{equation}
For each observable stabilizing controller $\mK$, Theorem~\ref{theorem.optimal_tansformation} implies that there exists an optimal transformation that leads to the lowest \texttt{dLQR} cost.  Fig. \ref{fig:example_1} demonstrates this fact. In particular, the red line of Fig. \ref{fig:example_1} displays the orbit of the similarity transformation of controller 
$$\mK = \begin{bmatrix}
0&-0.944\\
1.1&-0.944
\end{bmatrix}.$$ We can see that the \texttt{dLQR} cost changes with different similarity transformations, which also shows that finding the optimal similarity transformation (marked as the red point) can significantly improve the control performance. \hfill $\square$
\end{example}

\begin{example}
\label{example:2}
Consider an open-loop stable dynamic system \eqref{eq.statefunction} with 
$$A=0.9, \; B=1, \; C=1, \; Q=5, \; R=1.$$ 
According to~\cite[Theoerem D.4]{zheng2021analysis}, the set of stabilizing controllers $\mathbb{K}$ for this system is nonconvex but connected.  To define \texttt{dLQR}~\eqref{eq.dLQR}, we choose $X$ as~\eqref{eq:exampleX}. 
Again, for each observable stabilizing controller $\mK$, Theorem~\ref{theorem.optimal_tansformation} implies that there exists an optimal transformation, shown in Fig. \ref{fig:example_2}, where the red line displays the orbit of the similarity transformation of controller 
$$\mK = \begin{bmatrix}
0&-0.765\\
0.9&-0.765
\end{bmatrix}$$
and the red point represents the optimal similarity transformation. \hfill $\square$
\end{example}

\begin{figure}[t]
\centering
\captionsetup{singlelinecheck = false,labelsep=period, font=small}
\captionsetup[subfigure]{justification=centering}
\subfloat[]{\label{fig:example_1}\includegraphics[width=0.49\linewidth]{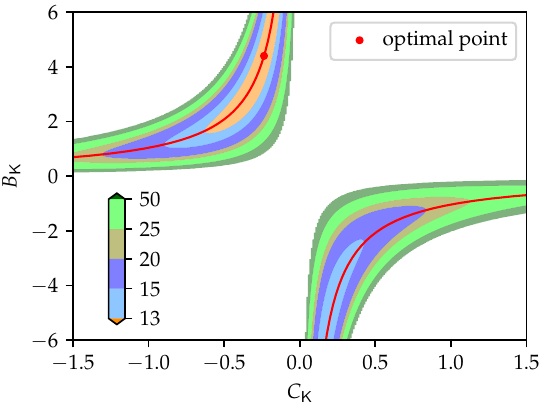}}
\subfloat[]{\label{fig:example_2}\includegraphics[width=0.49\linewidth]{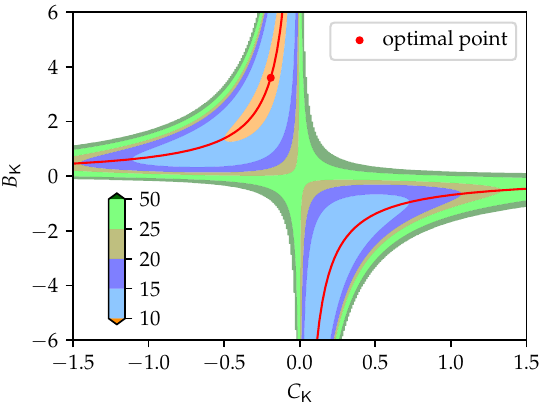}}
\caption{\texttt{dLQR} cost of Examples \ref{example:1} and  \ref{example:2}. (a) \texttt{dLQR} cost for system in Example \ref{example:1} when fixing $A_{\mK}=-0.944$. The red line represents all points in the set  $\{(B_{\mK},C_{\mK})|B_{\mK}=1.1T, C_{\mK}=-0.944/T, T\neq 0\}$. (b) \texttt{dLQR} cost for system in Example \ref{example:2} when fixing $A_{\mK}=-0.765$. The red line represents all points in the set  $\{(B_{\mK},C_{\mK})|B_{\mK}=0.9T, C_{\mK}=-0.765/T, T\neq 0\}$. }
\label{f:problem_1_cost}
\end{figure}

\section{Structure of Stationary Points}
\label{sec.mainresults}

In this section, we characterize the stationary points of Problem \ref{pro.dLQR} by letting the gradients of the \texttt{dLQR} cost be zero. Upon denoting the gradient of $J(\mK)$ w.r.t. $A_{\mK}$, $B_{\mK}$, and $C_{\mK}$ as $\nabla_{A_{\mK}}J(\mK)$, $\nabla_{B_{\mK}}J(\mK)$, and $\nabla_{C_{\mK}}J(\mK)$, respectively, we define the set of stationary points as 
$$
\mathbb{K}_s := \left\{ \begin{bmatrix}
0_{m\times d} & C_{\mK} \\
B_{\mK} & A_{\mK}
\end{bmatrix}:\left\|\begin{bmatrix}
0_{m\times d} & \nabla_{C_{\mK}} J(\mK) \\
\nabla_{B_{\mK}} J(\mK) & \nabla_{A_{\mK}} J(\mK)
\end{bmatrix}\right\|_F=0 \right\}.
$$  
We now look into the structure of $\mathbb{K}_s$, which is crucial for understanding the performance of policy gradient methods on \texttt{dLQR} problems. 

\begin{theorem}
\label{theorem.solution_expression}
Suppose $C$ has full row rank, $X \succ 0$, and Assumption \ref{assumption.control_observe} holds. If an observable stationary point, i.e., $\mK^\star \in \mathbb{K}_o \cap \mathbb{K}_s \cap \mathbb{K}$, to Problem \ref{pro.dLQR} exists, it is unique and in the form of 
\begin{equation}
\label{eq.optimal_form_K}
\mK^\star=\mathscr{T}_{T^\star}({\mK}^\ddagger),
\end{equation}
where 
\begin{equation}
\label{eq.riccati_K}
\mK^{\ddagger}:=\begin{bmatrix}
  0   & -K^\star  \\
  L^\star   & A-BK^\star-L^\star C
\end{bmatrix},
\end{equation}
$T^\star= X_{22}X_{12}^{-1}$ is the optimal transformation matrix of ${\mK}^\ddagger$ given in \eqref{eq.optimal_tranforsmation} with $
-P_{{\mK}^\ddagger,12}^{-\tr}P_{{\mK}^\ddagger,22}=I_n
$, and
\begin{equation}
\label{eq.optimal_L}
L^\star = A{\hat{\Sigma}}C^\tr (C{\hat{\Sigma}}C^\tr)^{-1},
\end{equation}
\begin{equation}
\label{eq.optimal_K}
K^\star= (R+B^\tr {\hat{P}}B)^{-1}B^\tr {\hat{P}}A,
\end{equation}
with $\hat{\Sigma}$ and ${\hat{P}}$ being the unique positive definite solutions to 
\begin{equation}
\label{eq.sigma_riccati}
{\hat{\Sigma}}=\Delta_X+A {\hat{\Sigma}} A^\tr - A {\hat{\Sigma}}C^\tr \left(C{\hat{\Sigma}}{C}^\tr\right)^{-1} C {\hat{\Sigma}} A^\tr,
\end{equation}
\begin{equation}
\label{eq.P_riccati}
{\hat{P}}=Q+ A^\tr {\hat{P}} A -A^\tr {\hat{P}} B (R+B^\tr {\hat{P}} B)^{-1}B^\tr {\hat{P}} A.
\end{equation}
\end{theorem}

In Theorem~\ref{theorem.solution_expression}, $\mK^\star$ is an elegant closed-form solution since it satisfies the optimal similarity transformation of a special observer-based controller $\mK^\ddagger$. Note that $K^\star$ of \eqref{eq.riccati_K} is exactly the optimal control gain of the state-feedback LQR and $L^\star$ is a stable observer gain. In classical control theory \cite{lewis2012optimal}, the observer-based controller of Problem \ref{pro.dLQR} can separate into a stable observer and a state-feedback LQR; however, the transient behavior induced by the initial state and estimate is not considered. As a comparison, both the observer gain $L^\star$ and the optimal transformation matrix ${T^\star}$ of the observable stationary point are uniquely determined according to the prior information of the initial distribution of system state and controller state $(x_0,\xi_0)\sim \bar{\mathcal{D}}$. 

Note that Theorem \ref{theorem.solution_expression} does not discuss the theoretical optimality of the identified stationary point, which will be of interest for future work. In practical applications, if the optimal controller of a given system is known to be observable, then $\mK^\star$ in~\eqref{eq.optimal_form_K} must be the globally optimal controller due to its uniqueness. For instance, the observable stationary points of Examples \ref{example:1} and \ref{example:2}, i.e.,
$$\mK_1^\star = \begin{bmatrix}
0&-0.236\\
4.4& -0.944
\end{bmatrix} \quad \text{and} \quad \mK_2^\star = \begin{bmatrix}
0&-0.191\\
3.6& -0.765
\end{bmatrix},$$ 
are globally optimal by Theorem \ref{theorem.solution_expression}. They agree with the exhausted numerical grid search for the globally optimal points (marked as red points in Fig. \ref{f:problem_1_cost}) in Examples \ref{example:1} and \ref{example:2}, respectively.

 \balance
\section{Conclusion}
\label{sec:conclusion}
In this paper, we have analyzed the policy gradient optimization landscape of linear quadratic control problems using dynamic output-feedback policies. We have shown that the \texttt{dLQR} cost varies with similarity transformations, and identified the structure of the optimal similarity transformation of an observable stabilizing controller. More importantly, we characterized the stationary point of the policy gradient optimization and proved that the associated dynamic controller is unique if it is observable. Our work brings new insights for understanding the policy gradient algorithms for solving the partially observed control or decision-making problems. 


\appendix

\subsection{Block-wise Lyapunov equations and useful lemmas}
The block-wise Lyapunov equations in \eqref{eq.lyapunov_equation} and \eqref{eq.lyapunov_equation_sigma} will be used extensively in this paper. From \eqref{eq.lyapunov_equation}, we have
\begin{subequations}
\begin{align}
&\begin{aligned}
\label{eq.block_lyapunov_P11}
P_{11}&=Q+ A^\tr P_{11} A +C^\tr B_{\mK}^\tr P_{12}^\tr A \\
&\qquad\qquad + A^\tr P_{12} B_{\mK} C+C^\tr B_{\mK}^\tr P_{22} B_{\mK} C,
\end{aligned}\\
&\begin{aligned}
\label{eq.block_lyapunov_P12}
P_{12}&=A^\tr P_{11}BC_{\mK} + C^\tr B_{\mK}^\tr P_{12}^\tr B C_{\mK} \\
&\qquad\qquad + A^\tr P_{12} A_{\mK} +C^\tr B_{\mK}^\tr P_{22} A_{\mK},
\end{aligned}\\
&\begin{aligned}
\label{eq.block_lyapunov_P22}
P_{22}&=C_{\mK}^\tr RC_{\mK}+ A_{\mK}^\tr P_{12}^\tr B C_{\mK} + C_{\mK}^\tr B^\tr P_{12} A_{\mK} \\
&\qquad\qquad +C_{\mK}^\tr B^\tr P_{11} B C_{\mK}+A_{\mK}^\tr P_{22}A_{\mK}.
\end{aligned}
\end{align}
\end{subequations}
Similarly,
we get
\begin{subequations}
\label{eq.block_lyapunov_sigma}
\begin{align}
&\begin{aligned}
\label{eq.block_lyapunov_sigma11}
\Sigma_{11}&=X_{11}+A \Sigma_{11} A^\tr +BC_{\mK} \Sigma_{12}^\tr A^\tr\\
&\qquad \qquad + A \Sigma_{12} C_{\mK}^\tr B^\tr+B C_{\mK} \Sigma_{22} C_{\mK}^\tr B^\tr,
\end{aligned}\\
&\begin{aligned}
\label{eq.block_lyapunov_sigma12}
\Sigma_{12}&=X_{12}+  A \Sigma_{11}C^\tr B_{\mK}^\tr + B C_{\mK} \Sigma_{12}^\tr C^\tr B_{\mK}^\tr \\
&\qquad \qquad + A \Sigma_{12} A_{\mK}^\tr +B C_{\mK} \Sigma_{22} A_{\mK}^\tr,
\end{aligned}\\
&\begin{aligned}
\label{eq.block_lyapunov_sigma22}
\Sigma_{22}&=X_{22}+B_{\mK} C \Sigma_{11} C^\tr B_{\mK}^\tr + A_{\mK}\Sigma_{12}^\tr C^\tr B_{\mK}^\tr \\
&\qquad \qquad + B_{\mK}C \Sigma_{12} A_{\mK}^\tr+A_{\mK} \Sigma_{22}A_{\mK}^\tr.
\end{aligned}
\end{align}
\end{subequations}

Standard Lyapunov theorems will be used throughout the paper. We summarize them below for completeness.
\begin{lemma}[Lyapunov stability theorems \cite{gu2012discrete,lee2018primal}]
\label{lemma.Lyapunov_stability}
\ 
\begin{enumerate}[(a)]
\item If $\rho(A) <1$ and $Q \in \mathbb{S}_{+}^n$, the Lyapunov equation $P=Q+A^\tr P A$ has a unique solution $P \in \mathbb{S}_{+}^n$.
\item Let $Q \in \mathbb{S}_{++}^n$. $\rho(A) <1$ if and only if there exists a unique $P \in \mathbb{S}_{++}^n$ such that $P=Q+A^\tr P A$. 
\item Suppose $(C,A)$ is observable.  $\rho(A) <1$ if and only if there exists a unique $P \in \mathbb{S}_{++}^n$ such that $P=C^\tr C+A^\tr P A$. 
\end{enumerate}
\end{lemma}

Given an observable stabilizing controller, the following lemma is a discrete-time counterpart to~\cite[Lemma 4.5]{zheng2021analysis}.
\begin{lemma}
\label{lemma.positive_P}
Under Assumption \ref{assumption.control_observe}, if $\mK\in \mathbb{K}\cap\mathbb{K}_o$, the solution $P_{\mK}$ to \eqref{eq.lyapunov_equation} is unique and positive definite.
\end{lemma}

\subsection{Non-existence of the optimal similarity transformation}
\label{appendix.non-existence}

We take a one-dimensional system as an example (i.e., $x_t$ and $\xi_t$ are scalars), to show the non-existence of the optimal similarity transformation if $X_{12}$ is singular. 
Under similarity transformation~\eqref{eq:similarity-transformation}, we have 
\begin{equation}
\label{eq.u_sequence}
\begin{aligned}
u_0 &= C_{\mK}T^{-1}\xi_0, \\
u_1 &= C_{\mK}T^{-1}( TA_{\mK}T^{-1}\xi_0 + TB_{\mK}y_0) \\
& = C_{\mK}A_{\mK}T^{-1}\xi_0 + C_{\mK}B_{\mK}y_0.
\end{aligned}
\end{equation}
Given an observable stabilizing controller $\mK$, by \eqref{eq.optimal_tranforsmation} of  \Cref{theorem.optimal_tansformation}, one has 
\begin{equation}
\label{eq.T_limit}
\lim_{X_{12}\rightarrow 0} (T^\star)^{-1} = \lim_{X_{12}\rightarrow 0} - P_{22}^{-1}P_{12}^\tr X_{12}X_{22}^{-1} =  0.
\end{equation}
Note that the cross-correlation value $X_{12}=0$ if the initial controller state $\xi_0$ is zero-mean and independent of the initial system state $x_0$. Using \eqref{eq.T_limit} in \eqref{eq.u_sequence}, we can observe that the controller input $u_t$ tends to ignore the influence of $\xi_0$ by increasing $T$ in this one-dimensional instance. This is because that the initial controller state provides no information for the estimation of the initial system state if $X_{12}=0$.

\subsection{Proof of Proposition \ref{prop.similarity_variance}}
\begin{proof}
Since $\mK\in \mathbb{K}$, by Lemma \ref{lemma.Lyapunov_stability}(a), the Lyapunov equation \eqref{eq.lyapunov_equation} admits a unique positive semidefinite solution for both $\mK$ and $\mathscr{T}_T(\mK)$. Hence, the solution of \eqref{eq.lyapunov_equation} for $\mK$ can be expressed as 
\begin{equation}
\label{eq.P_sum}
P_{\mK}= \sum_{k=0}^{\infty}\left((\bar{A}+\bar{B}\mK\bar{C})^\tr\right)^k \begin{bmatrix}
  Q  & 0 \\
  0  &  C_{\mK}^\tr RC_{\mK}
\end{bmatrix}(\bar{A}+\bar{B}\mK\bar{C})^k.  
\end{equation}
Similarly, by the definition of $\mathscr{T}_T(\mK)$ in~\eqref{eq:similarity-transformation}, one has
\begin{equation}
\label{eq.similarity_P}
P_{\mathscr{T}_T(\mK)}= \bar{T}^{-\tr}P_{\mK}\bar{T}^{-1}.
\end{equation}
Therefore, by \eqref{eq.cost_in_P}, we have
\begin{equation}
\nonumber
J(\mathscr{T}_T(\mK)) = {\rm Tr}\left(P_{\mathscr{T}_T(\mK)}X\right)  = {\rm Tr}\left(P_{\mK}\bar{T}^{-1}X\bar{T}^{-\tr}\right),
\end{equation}
which completes the proof.
\end{proof}

\subsection{Proof of Theorem \ref{theorem.optimal_tansformation}}
\begin{proof}
By \eqref{eq.cost_transformation}, $J(\mathscr{T}_T(\mK))$ can be expressed as
\begin{equation}
\label{eq.cost_of_similarity}
\begin{aligned}
J(\mathscr{T}_T(\mK))&={\rm Tr}\left(P_{11}X_{11}+P_{12}T^{-1}X_{12}^\tr \right.\\
&\qquad \left. +P_{12}^\tr X_{12}T^{-\tr}+P_{22}T^{-1}X_{22}T^{-\tr}\right).
\end{aligned}
\end{equation}
For notational convenience, given a stabilizing controller $\mK\in \mathbb{K}$, we denote the cost value $J(\mathscr{T}_T(\mK))$ w.r.t. similarity transformation $T$ as 
$$
g(\mH) := J(\mathscr{T}_T(\mK)), \quad \text{with} \; \mH := T^{-1}\in \mathrm{GL}_n.
$$
It is clear that $g(\mH)$ is twice differentiable w.r.t. $\mH$. 
The gradient of $g(\mH)$ w.r.t. $\mH$ can be derived as
\begin{equation}
\label{eq.gradient_of_gH}
\begin{aligned}
\nabla_{\mH}g(\mH)=2(P_{12}^\tr X_{12}+P_{22}\mH X_{22}).
\end{aligned}
\end{equation}
By Lemma \ref{lemma.positive_P}, the solution $P_{\mK}$ to \eqref{eq.lyapunov_equation} is positive definite, which means $P_{22}$ is invertible. We also have that $X_{22}$ is invertible since $X \succ 0$. Let $\nabla_{\mH}g(\mH)=0$, we have 
$$
    {\mH}^\star = - P_{22}^{-1}P_{12}^\tr X_{12}X_{22}^{-1}. 
$$
By $(T^\star)^{-1}={\mH}^\star$, we now identify $T^\star$ is in the form of~\eqref{eq.optimal_tranforsmation}. This also implies that if $T^\star$ exists, both $X_{12}$ and $P_{12}$ must be invertible.

Next, we show that $T^\star$ in~\eqref{eq.optimal_tranforsmation} is the unique globally optimal similarity transformation matrix such that \eqref{eq.optimal_T} holds. We analyze the  Hessian of $g(\mH)$ applied to a nonzero direction $Z \in \mathbb{R}^{n\times n}$, which is  
\begin{equation} 
\nonumber
\nabla^2g(\mH)[Z,Z]:=\frac{d^2}{d\eta^2}\Big|_{\eta=0} g(\mH + \eta Z).
\end{equation}
By \eqref{eq.cost_of_similarity}, we can further show that 
\begin{equation} 
\nonumber
\begin{aligned}
\,&\nabla^2g(\mH)[Z,Z] \\ 
=\,&\frac{d^2}{d\eta^2}\Big|_{\eta=0}{\rm Tr}(P_{12}(\mH+\eta Z)X_{12}^\tr+P_{12}^\tr X_{12}(\mH+\eta Z)^\tr\\ 
&\qquad\qquad\qquad\quad+P_{22}(\mH+\eta Z)X_{22}(\mH+\eta Z)^\tr)\\
=\,&2{\rm Tr}(P_{22}ZX_{22}Z^\tr)\\
\geq\,& 2\lambda_{\rm min}(P_{22})\lambda_{\rm min}(X_{22})\|Z\|_F^2\\
>\,& 0.
\end{aligned}
\end{equation}
We extend the function $g(\mH)$ to be defined on a convex superset $\mathbb{R}^{n\times n}$ of $\mathrm{GL}_n$. 
It is immediate that $g(\mH)$ is strongly convex over $\mathbb{R}^{n\times n}$, which means the globally optimum of $g(\mH)$ over $\mathrm{GL}_n$ is unique when it exists. By $T = \mH^{-1}$, then the globally optimum of $J(\mathscr{T}_T(\mK))$ is also unique over $T \in \mathrm{GL}_n$, thus \eqref{eq.optimal_T} is satisfied for a unique $T^*$.
\end{proof}

\subsection{Policy Gradient Expression}
\begin{lemma}[Policy Gradient Expression]
\label{lemma:gradient}
 For $\forall  \mK\in \mathbb{K}$, the policy gradient of Problem \ref{pro.dLQR} is 
\begin{subequations}
\label{eq.gradient}
\begin{align}
&\begin{aligned}
\label{eq.gradient_k12}
\nabla_{C_{\mK}} &J(\mK)=2B^\tr (P_{11}A+P_{12}B_{\mK}C)\Sigma_{12}\\
&+2((R + B^\tr P_{11} B) C_{\mK}+B^\tr P_{12}A_{\mK})\Sigma_{22},
\end{aligned}\\
&\begin{aligned}
\label{eq.gradient_k21}
\nabla_{B_{\mK}} J(\mK)
&= 2(P_{12}^\tr A +P_{22}B_{\mK}C)\Sigma_{11} C^\tr \\
&\quad + 2(P_{12}^\tr B C_{\mK}+P_{22}A_{\mK})\Sigma_{12}^\tr C^\tr,
\end{aligned}\\
&\begin{aligned}
\label{eq.gradient_k22}
\nabla_{A_{\mK}} J(\mK)&= 2(P_{12}^\tr B C_{\mK}+P_{22}A_{\mK})\Sigma_{22}\\
&\quad + 2 (P_{12}^\tr A+P_{22}B_{\mK}C)\Sigma_{12}.
\end{aligned}
\end{align}
\end{subequations}
\end{lemma}
\begin{proof}
The proof follows the similar lines as the state-feedback LQR case~\cite[Lemma 1]{fazel2018global}. By \eqref{eq.lyapunov_equation}, the value function of $\bar{x}_0$ reads as 
\begin{equation}
\nonumber
\begin{aligned}
V_{\mK}(\bar{x}_0) 
= \;&\bar{x}_0^\tr P_{\mK}\bar{x}_0\\
= \;&\bar{x}_0^\tr (\bar{Q} + F^\tr C_{\mK}^\tr RC_{\mK}F)\bar{x}_0\\
&\qquad \quad +\bar{x}_0^\tr(\bar{A}+\bar{B}\mK\bar{C})^\tr P_{\mK}(\bar{A}+\bar{B}\mK\bar{C})\bar{x}_0\\
=\;&\bar{x}_0^\tr (\bar{Q} + F^\tr C_{\mK}^\tr RC_{\mK}F)\bar{x}_0+V_{\mK}((\bar{A}+\bar{B}\mK\bar{C})\bar{x}_0).
\end{aligned}
\end{equation}
Taking the gradient of $V_{\mK}(\bar{x}_0)$ w.r.t. $C_{\mK}$ (note that $\nabla_{C_{\mK}} V_{\mK}((\bar{A}+\bar{B}\mK\bar{C})\bar{x}_0)$ has two terms: one with respect to $C_{\mK}$ in the subscript and one with respect to the input $(\bar{A}+\bar{B}\mK\bar{C})\bar{x}_0$), we have
\begin{equation}
\nonumber
\begin{aligned}
\nabla_{C_{\mK}} V_{\mK}(\bar{x}_0) 
= \; &2((R + B^\tr P_{11} B) C_{\mK}+B^\tr P_{12}A_{\mK})\xi_0\xi_0^\tr\\
&\qquad + 2B^\tr (P_{11}A+P_{12}B_{\mK}C)x_0 \xi_0^\tr \\
&\qquad +\bar{x}_1^\tr \nabla_{C_{\mK}} P_{\mK} \bar{x}_1\big|_{\bar{x}_1=(\bar{A}+\bar{B}\mK\bar{C})\bar{x}_0}\\
= \; &2((R + B^\tr P_{11} B) C_{\mK}+B^\tr P_{12}A_{\mK})\sum_{t=0}^{\infty}\xi_t \xi_t^\tr\\
&\qquad + 2B^\tr (P_{11}A+P_{12}B_{\mK}C)\sum_{t=0}^{\infty}x_t \xi_t^\tr,
\end{aligned}
\end{equation}
where the last step uses recursion and that $x_{t+1} = (\bar{A}+\bar{B}\mK\bar{C})\bar{x}_t$. 

We can also derive the formulas of $\nabla_{B_{\mK}} V_{\mK}(\bar{x}_0)$ and $\nabla_{A_{\mK}} V_{\mK}(\bar{x}_0) $ through
similar steps. Then, we can finally observe \eqref{eq.gradient} by taking the expectation w.r.t. the initial distribution $
\bar{\mathcal{D}}$.
\end{proof}
\subsection{Proof of Theorem \ref{theorem.solution_expression}}
\begin{proof}
Suppose an observable stationary point exists, denoted as $\mK^\star \in \mathbb{K}_o \cap \mathbb{K}_s \cap \mathbb{K}$. By Lemma \ref{lemma.Lyapunov_stability}(b) and Lemma \ref{lemma.positive_P}, we know $\Sigma_{\mK^\star}, P_{{\mK}^\star} \in \mathbb{S}_{++}^{2n}$. By the Schur complement,  it is obvious that
\begin{equation}
\nonumber
\begin{aligned}
{\hat{P}}&:=P_{11}-P_{12}P_{22}^{-1}P_{12}^\tr \in \mathbb{S}^n_{++},\\  {\hat{\Sigma}}&:=\Sigma_{11}  - \Sigma_{12}\Sigma_{22}^{-1}\Sigma_{12}^\tr \in \mathbb{S}^n_{++}.
 \end{aligned}
\end{equation}

Throughout this proof, the subscript of the submatrices of $\Sigma_{\mK^\star}$ and  $P_{{\mK}^\star}$ under observable stationary point $\mK^\star$ will be omitted. Since \eqref{eq.gradient} is linear in $A_{\mK}$, $B_{\mK}$, and $C_{\mK}$, when  $\mK^\star \in \mathbb{K}_s$, it is not hard to show that 
\begin{subequations}
\label{eq.K_original}
\begin{align}
&\begin{aligned}
\label{eq.K_12_original}
C_{\mK^\star}& =-K^\star\Sigma_{12}\Sigma_{22}^{-1},
\end{aligned}\\
&\begin{aligned}
\label{eq.K_21_original}
B_{\mK^\star} = -P_{22}^{-1} P_{12}^\tr L^\star,
\end{aligned}\\
&\begin{aligned}
\label{eq.K_22_original}
A_{\mK^\star}=-P_{22}^{-1}P_{12}^\tr (A-L^\star C-BK^\star)\Sigma_{12}\Sigma_{22}^{-1},
\end{aligned}
\end{align}
\end{subequations}
where $K^\star$ and $L^\star$ are
$$
\begin{aligned}
K^\star &= (R + B^\tr{\hat{P}} B)^{-1}B^\tr {\hat{P}}A, \\
L^\star &= A{\hat{\Sigma}} C^\tr (C{\hat{\Sigma}} C^\tr)^{-1}.
\end{aligned}$$

Combining  \eqref{eq.block_lyapunov_P12}, \eqref{eq.block_lyapunov_P22}, and \eqref{eq.K_original}, we prove that 
\begin{equation}
\label{eq.inverse}
 P_{12}^\tr \Sigma_{12} + P_{22}\Sigma_{22}=0,  
\end{equation}
which immediately leads to \begin{equation}
\label{eq.inverse_relation}
(-P_{22}^{-1}P_{12}^\tr)^{-1}=\Sigma_{12}\Sigma_{22}^{-1}.
\end{equation}
We then define $T^\ddagger:=-P_{22}^{-1}P_{12}^\tr$, and thus $(T^\ddagger)^{-1}=\Sigma_{12}\Sigma_{22}^{-1}$. 
Similarly, from \eqref{eq.block_lyapunov_sigma12}, \eqref{eq.block_lyapunov_sigma22}, and \eqref{eq.K_original}, \eqref{eq.inverse} can be rewritten as 
\begin{equation}
\label{eq.relation_P_and_x}
P_{12}^\tr X_{12} + P_{22}X_{22}=0.
\end{equation}
Combining \eqref{eq.inverse} with \eqref{eq.relation_P_and_x} leads to
\begin{equation}
\label{eq.T_and_X}
{T^\ddagger}= -P_{22}^{-1}P_{12}^\tr=X_{22}X_{12}^{-1}.
\end{equation}
Combining \eqref{eq.K_original}, \eqref{eq.inverse_relation}, and \eqref{eq.T_and_X}, we can observe that $\mK^\star$ is in the form shown in \eqref{eq.optimal_form_K}.

It remains to show that 
\begin{itemize}
    \item ${T^\ddagger}$ is the optimal transformation matrix of ${\mK}^\ddagger$ given in \eqref{eq.optimal_tranforsmation} (i.e., $T^\ddagger=- X_{22}X_{12}^{-1}P_{{\mK}^\ddagger,12}^{-\tr}P_{{\mK}^\ddagger,22}=T^\star$);
    \item ${\hat{P}}$ and ${\hat{\Sigma}}$ are the unique positive definite solutions to the Riccati equations \eqref{eq.P_riccati} and \eqref{eq.sigma_riccati}, respectively. 
\end{itemize}
First,  by \eqref{eq.optimal_form_K} and \eqref{eq.similarity_P} in Proposition \ref{prop.similarity_variance}, we have
\begin{equation}
\nonumber
P_{{\mK}^\star}=\begin{bmatrix}
I_n & 0 \\
0 & (T^\ddagger)^{-\tr}
\end{bmatrix}P_{{\mK}^\ddagger}\begin{bmatrix}
I_n & 0 \\
0 & (T^\ddagger)^{-1}
\end{bmatrix}.
\end{equation}
From \eqref{eq.relation_P_and_x}, it is not hard to show that 
\begin{equation}
\nonumber
(T^\ddagger)^{-\tr}P_{{\mK}^\ddagger,12}^\tr X_{12} + (T^\ddagger)^{-\tr}P_{{\mK}^\ddagger,22}(T^\ddagger)^{-1}X_{22}=0,
\end{equation}
which directly leads to 
$
-P_{{\mK}^\ddagger,22}^{-1}P_{{\mK}^\ddagger,12}^\tr=I_n.
$
Therefore, by \eqref{eq.optimal_tranforsmation} of Theorem \ref{theorem.optimal_tansformation}, one has
\begin{equation}
\nonumber
{T^\ddagger}=X_{22}X_{12}^{-1}=- X_{22}X_{12}^{-1}P_{{\mK}^\ddagger,12}^{-\tr}P_{{\mK}^\ddagger,22}=T^\star,
\end{equation}
which is exactly the optimal transformation matrix of ${\mK}^\ddagger$.

Then, we will derive \eqref{eq.P_riccati}.  Multiplying \eqref{eq.block_lyapunov_P22} by ${T^\star}^\tr$ on the left and by ${T^\star}$ on the right (or multiplying \eqref{eq.block_lyapunov_P12} by ${T^\star}$ on the right), we have
\begin{equation}
\label{eq.P_22_transform}
\begin{aligned}
&P_{12}P_{22}^{-1}P_{12}^\tr= A^\tr {\hat{P}} B (R+B^\tr {\hat{P}} B)^{-1}B^\tr {\hat{P}} A \\
&\quad +A^\tr P_{12}P_{22}^{-1}P_{12}^\tr A+C^\tr {L^\star}^\tr P_{12}P_{22}^{-1}P_{12}^\tr L^\star C\\
&\quad -A^\tr P_{12}P_{22}^{-1}P_{12}^\tr L^\star C-C^\tr {L^\star}^\tr P_{12}P_{22}^{-1}P_{12}^\tr A.
\end{aligned}
\end{equation}
Then,  plugging \eqref{eq.K_21_original} in \eqref{eq.block_lyapunov_P11} leads to
\begin{equation}
\label{eq.P_11_expand}
\begin{aligned}
P_{11}&=Q+ A^\tr P_{11} A - C^\tr {L^\star}^\tr P_{12} P_{22}^{-1} P_{12}^\tr A \\
&- A^\tr P_{12} P_{22}^{-1}P_{12}^\tr L^\star C+C^\tr {L^\star}^\tr P_{12}  P_{22}^{-1}P_{12}^\tr L^\star C.
\end{aligned}
\end{equation}
Subtracting \eqref{eq.P_22_transform} from \eqref{eq.P_11_expand}, we can finally see that $\hat{P}$ satisfies the Riccati equation \eqref{eq.P_riccati}. Through similar steps, we can derive from \eqref{eq.block_lyapunov_sigma} that $\hat{\Sigma}$ satisfies the Riccati equation \eqref{eq.sigma_riccati}, which completes the proof.
\end{proof}

\bibliographystyle{ieeetr}
\bibliography{ref}

\end{document}